\documentclass{ecai} 

\usepackage{latexsym}
\usepackage{amssymb}
\usepackage{amsmath}
\usepackage{amsthm}
\usepackage{booktabs}
\usepackage{enumitem}
\usepackage{graphicx}
\usepackage{color}
\usepackage{subcaption}

\usepackage{algorithm}
\usepackage{algorithmic}
\usepackage[switch]{lineno}
\usepackage{multirow}
\usepackage{multicol}
\usepackage{caption}

\newtheorem{theorem}{Theorem}
\newtheorem{lemma}[theorem]{Lemma}

\newtheorem{proposition}[theorem]{Proposition}

\newtheorem{remark}[theorem]{Remark}

\newtheorem{definition}[theorem]{Definition}
\newtheorem{problem}[theorem]{Problem}

\newcommand{\BibTeX}{B\kern-.05em{\sc i\kern-.025em b}\kern-.08em\TeX}

\begin{document}


\begin{frontmatter}


\paperid{123} 


\title{Improving Reinforcement Learning Sample-Efficiency using Local Approximation}


\author[A]{\fnms{Mohit}~\snm{Prashant}}
\author[B]{\fnms{Arvind}~\snm{Easwaran}}


\begin{abstract}
In this study, we derive Probably Approximately Correct (PAC) bounds on the asymptotic sample-complexity for RL within the infinite-horizon Markov Decision Process (MDP) setting that are sharper than those in existing literature. The premise of our study is twofold: firstly, the further two states are from each other, transition-wise, the less relevant the value of the first state is when learning the $\epsilon$-optimal value of the second; secondly, the amount of 'effort', sample-complexity-wise, expended in learning the $\epsilon$-optimal value of a state is independent of the number of samples required to learn the $\epsilon$-optimal value of a second state that is a sufficient number of transitions away from the first. Inversely, states within each other's vicinity have values that are dependent on each other and will require a similar number of samples to learn. By approximating the original MDP using smaller MDPs constructed using subsets of the original's state-space, we are able to reduce the sample-complexity by a logarithmic factor to $O(SA \log A)$ timesteps, where $S$ and $A$ are the state and action space sizes. We are able to extend these results to an infinite-horizon, model-free setting by constructing a PAC-MDP algorithm with the aforementioned sample-complexity. We conclude with showing how significant the improvement is by comparing our algorithm against prior work in an experimental setting.
\end{abstract}

\end{frontmatter}


\section{Introduction}
The field of reinforcement learning (RL) has made rapid progress in recent years, with studies demonstrating its application in complex decision making tasks \cite{23}. As a consequence, several frameworks have been proposed to build complex decision making systems utilizing RL in domains like robotics, healthcare and autonomous transport \cite{22}. However, to ensure trust in these systems, it is necessary to evaluate how successful a learner is during the learning process. Noting this, our aim is to derive the number of timesteps required to learn an $\epsilon$-optimal policy, where $\epsilon \in (0, 1)$ is the distance of the learnt policy from the best possible policy and a timestep is an interaction of the learner with the environment.

A body of work that addresses this question is on Probably Approximately Correct Markov Decision Processes (PAC-MDP). They are a class of algorithms that are applicable to discrete state-action MDPs and are known to provide PAC guarantees on the sample-complexity; that is, they guarantee $\epsilon$-optimal solutions with a specific confidence, $1-\delta$, while having a sample-complexity that is polynomial in the environment size. As it is infeasible to obtain transition and reward models in many domains \cite{4}, we focus on computing sample-complexity bounds for model-free approaches to RL problems.



The current SOTA sample-complexity bound proposed for RL algorithms is $O(SA \log (SA))$ timesteps, where $S$ and $A$ are the state and action space sizes. This sample-complexity is demonstrated by algorithms like Delayed Q-Learning \cite{16}, Variance Reduced Q-Learning \cite{27} and Q-Learning with UCB Exploration \cite{28} amongst other studies \cite{29, 30}. The difficulty in deriving sample-complexity bounds in model-free RL is that the lack of knowledge regarding the underlying MDP results in statistical bounds placed on the learning convergence to be conservative in nature. Citing Delayed Q-Learning and Variance Reduced Q-Learning as examples, the state-action space is over-sampled in trying to bound the variance of updates \cite{28}.

The aim of this study is to present sharper PAC-bounds on the sample-complexity of reinforcement learning within environments endowed with a distance metric. Intuitively, the premise of our approach is two-fold. \textbf{(1)} Firstly, the further two states are from each other, the less relevant the value of the first state is in learning the $\epsilon$-optimal value of the second. Consequently, there exists a distance, dependent on the value of $\epsilon$, beyond which the values of the two states are effectively independent, with high confidence. \textbf{(2)} Secondly, the amount of 'effort' expended to learn the $\epsilon$-optimal value of a state, in terms of interactions the learner has with it, is independent of the effort required to learn the $\epsilon$-optimal value of a sufficiently distant state. The inverse of this is that the amount of effort expended in learning the value of a state is similar to the effort required to learn the values of its transition neighbors.

\begin{table*}[ht]
\captionsetup{justification=centering}
  \caption{\footnotesize A comparison of model-free RL algorithms in terms of sample-complexity for infinite-horizon Markov Decision Processes \\ 
  }
  \footnotesize
  \label{sample-table}
  \centering
  \begin{tabular}{|l|l|}
    \toprule
    Algorithm  & Sample-Complexity \\
    \midrule
    Delayed Q-Learning \cite{16}  & $O \left( \frac{SA}{\epsilon^4 (1-\gamma)^8}  \log(\frac{SA}{\delta\epsilon(1-\gamma)}) \log(\frac{1}{\delta}) \log(\frac{1}{\epsilon(1-\gamma)}) \right)$     \\ \hline 
    
    Speedy Q-Learning \cite{35} & $O \left(\frac{SA}{\epsilon^2 (1-\gamma)^4} 
 \log( \frac{SA}{\delta} ) \right)$      \\ \hline 
    
    Variance Reduced Q-Learning \cite{27}       & $O \left( \frac{SA}{\epsilon^2 (1-\gamma)^3}  \log(\frac{SA}{\delta(1-\gamma)}) \log(\frac{1}{\epsilon})  \right)$  \\ \hline 
    
    Q-Learning with UCB \cite{28}       & $O \left( \frac{SA}{\epsilon^2 (1-\gamma)^7} \log(SA) \log(\frac{1}{\delta}) \log(\frac{1}{\epsilon})  \log(\frac{1}{1-\gamma}) \right)$  \\ \hline 
    
    UCB-Multistage-Advantage \cite{36}       & $O \left( \frac{SA}{\epsilon^2 (1-\gamma)^{5.5}} \log(SA) \log(\frac{1}{\delta}) \log(\frac{1}{\epsilon(1-\gamma)}) \right)$  \\ \hline 
    
    Phased Q-Learning \cite{37}       & $O \left( \frac{SA}{\epsilon^2}  \log(\frac{SA}{\delta}\log(\frac{1}{\epsilon}))  \log({\frac{1}{\epsilon})} \right)$  \\ \hline

    Probabilistic Delayed Q-Learning \textbf{(this work)} & $O \left( \frac{SA}{\epsilon^3 (1-\gamma)^3} \log \left( \frac{A}{\delta(1-\gamma)} \right)  \log \frac{1}{\epsilon} \log \frac{1}{\delta})  \right)$ \\
    \bottomrule \hline
  \end{tabular}
\end{table*}

\vspace{-0.3cm}
\subsection{Our Contributions}

We make fundamental improvements on the sample-complexity bounds for model-free PAC-MDPs by approximating a larger MDP using several smaller MDPs constructed using subsets of the original state-space. With this, we implement a new model-free PAC-MDP algorithm, Probabilistic Delayed Q-Learning (PDQL), and prove its asymptotic sample-complexity bound is $O \left( \frac{SA}{\epsilon^3 (1-\gamma)^3} \log \left( \frac{A}{\delta(1-\gamma)} \right)  \log \frac{1}{\epsilon} \log \frac{1}{\delta}) \right)$. The value of our work lies in the following contributions.

\begin{enumerate}
    \item \textbf{We eliminate a logarithmic dependency on state-space size}, $\log S$, from prior SOTA model-free PAC-MDP convergence bounds, significantly increasing the learning rate when generalizing $\epsilon$-optimal value functions over larger state-action spaces or sparse reward settings.
    \item We experimentally compare PDQL with prior work in terms of convergence rates using benchmark environments and show it generalizes $\epsilon$-optimal solutions faster.
\end{enumerate}

\section{Related Works}


\sloppy The literature on model-free PAC-MDP algorithms is centered around Q-Learning \cite{34}. It should be noted that an extensive amount of analysis has been conducted on Q-Learning and various studies have improved on the base algorithm to improve sample efficiency in experimental settings \cite{29, 52, 53}. Though, with regard to works that have established upper bounds on the sample-complexity of learning $\epsilon$-optimal value functions, one of the earliest studies to derive PAC bounds for model-free RL is Delayed Q-Learning (DQL) \cite{16}. The principle behind this work is to stabilize training using a sampling procedure for Q-value update. This presents conditions for optimality of each update for which the corresponding mathematical analysis evaluates the probability of violation. The sample-complexity bound derived for this algorithm is $O \left( \frac{SA}{\epsilon^4 (1-\gamma)^8}  \log(\frac{SA}{\delta\epsilon(1-\gamma)}) \log(\frac{1}{\delta}) \log(\frac{1}{\epsilon(1-\gamma)}) \right)$, where $\gamma$ is the discount factor. This bound is noted to be quite conservative as it depends on an infinite-length execution of the algorithm, where the transitions from each state-action are not necessarily observed with equal frequency \cite{28, 37}. 

Under the assumption of access to a generative model/simulator, the Variance Reduced Q-Learning (VRQL) algorithm presents a sample-complexity of $O \left( \frac{SA}{\epsilon^2 (1-\gamma)^3}  \log(\frac{SA}{\delta(1-\gamma)}) \log(\frac{1}{\epsilon})  \right)$ \cite{27}. The underlying principle of the algorithm is to bound the variance of each Q-value update using a batched process, similar to \cite{16}. However, the assumption of access to a generative model in this work is able to reduce the sample complexity by a factor of $\epsilon^2 (1-\gamma)^5$ as the implication of this assumption is that the current interaction with the environment can be made independently of the previous interaction, similar to an oracle. This leads to an unbiased sampling of the Q-values in the state-action space and a more stable update than in Delayed Q-Learning. Other notable works that posit similar bounds are Q-Learning with UCB Exploration \cite{28} and Speedy Q-Learning \cite{35}. To the best of our knowledge, all prior work in the derivation of PAC guarantees for model-free RL show that the sample-complexity of learning an $\epsilon$-optimal value function is asymptotically bounded by $O(SA \log(SA))$ timesteps. This matches the known information theoretic bounds for infinite-horizon MDP problems presented in \cite{29, 51, 27}. We present a table of known relevant sample-complexity bounds for PAC-MDPs in Table \ref{sample-table} for comparison.


Our study builds upon the results presented by DQL and VRQL by making the assumption that learning takes place within a discrete environment endowed with a distance metric between states. By utilizing episodes of a finite length as opposed to an infinite-length execution and directing the agent to states that require value updates, we are able to guarantee faster convergence. Furthermore, in doing so, we are also able to reduce the number of samples required to meet the optimality conditions posited in \cite{27} by a factor of $\log S$. 



\section{Problem Formalization and Preliminaries}
\subsection{Notation}
\noindent In this study, the learning guarantee is denoted using error and confidence parameters, $\epsilon, \delta \in (0, 1)$, where the learner is able to generalize a policy with at most $\epsilon$ error with $1-\delta$ confidence. The environment that the learning problem is defined over is a discrete, finite-state MDP, $M$, that is characterized by the tuple, $(\mathcal{S}, \mathcal{A}, \mathcal{T}, \mathcal{R}, \gamma)$; where the state-space, $\mathcal{S}$, of $M$ has a size of $S$; the action-space of $M$, $\mathcal{A}$, has a size of $A$, i.e. the maximum number of actions available to a single state in $M$ is $A$; the underlying transition function, $\mathcal{T} : \mathcal{S} \times \mathcal{S} \times \mathcal{A} \rightarrow [0, 1]$, is a function that determines the transition probability to a state given a state-action in $M$; $\mathcal{R} : \mathcal{S} \times \mathcal{A} \rightarrow [0, 1]$ is a function specifying the reward for executing an action from a state; lastly $\gamma \in (0, 1)$ is the discount factor for future rewards. As the learning problem is model-free, note that $\mathcal{T}$ and $\mathcal{R}$ are unknown to the learner. Furthermore, note that the transition from executing an action on a state is stochastic and, therefore, the state-action reward is also stochastic. Hence, for some $(s, a) \in \mathcal{S} \times \mathcal{A}$, $\mathcal{R}(s, a)$ is the expected reward of the state-action. For the purpose of clarity, assume there also exists a reward function, $R : \mathcal{S} \rightarrow [0, 1]$, that denotes the true reward for visiting a state, unlike $\mathcal{R}$.

The value of a state, $s \in \mathcal{S}$, is equal to the expected cumulative discounted rewards experienced over a transition walk from $s$ for an arbitrary policy that determines the walk. Assuming the rewards experienced by a walk from $s$ over an infinite-horizon are described by the set $\{ r_1, r_2 ... \}$, the value of $s$ is the expectation over all possible walks of the discounted sum of the set, i.e. $V(s) := \mathbb{E} \left[ \sum_{i=0}^{\infty} \gamma^{i}r_{i} \right]$. Similarly, the value of a state evaluated over a fixed interval of $T$-steps, for some positive integer $T$, is $V^T(s) := \mathbb{E} \left[ \sum_{i=0}^{T} \gamma^{i}r_{i} \right]$. This study assumes a greedy policy is implemented for the learner to follow, recursively defining the value function as follows, where $s'$ is the subsequent state following the application of $a$ on $s$. The Q-value of a state-action is defined similarly. Note that the value function is a maximization of the Q-value over $\mathcal{A}$, i.e. $V(s) := \max_{a \in \mathcal{A}} \left( Q(s, a) \right)$. We denote the optimal value and Q-value functions over $M$ as $V_*(s)$ and $Q_*(s, a)$, respectively. 


\subsection{Problem Formalization}
\noindent Assume there exists an optimal value function defined over $M$'s state-space, $V_*(s)$, and a corresponding optimal Q-Value function defined over $M$'s state-action space, $Q_*(s, a)$, such that the application of a greedy policy from any state or state-action maximizes the expected cumulative discounted reward of transition walks enacted by the greedy policy. Letting $t$ be a positive integer, let $V_t(s)$ and $Q_t(s, a)$ be the value and Q-value function approximations learnt by an RL algorithm at timestep $t$. As the class of RL algorithms that we consider in this study are model-free PAC-MDP algorithms, the following definitions are used to provide constraints over the design of a model-free PAC-MDP algorithm \cite{5, 29, 31}.

\begin{definition} \textbf{(PAC-MDP Sample-Complexity)} \label{d1} \cite{5}
    For $\epsilon, \delta \in (0,1)$ and $t \in \mathbb{Z}^+$, a PAC-MDP algorithm applied to MDP $M$ is an algorithm that is able to generalize an $\epsilon$-optimal value function, $V_t(s)$, such that $\mathbb{P}(|V_t(s) - V_*(s)| < \epsilon) > 1-\delta, \forall s \in \mathcal{S}$ following $t$ timesteps; where $t$ is bounded by a polynomial in $S$ and $A$, the state-space and action-space sizes of $M$.
\end{definition}

\begin{definition} \textbf{(Model-Free Space-Complexity)} \label{d2} \cite{29}
    Any model-free learning algorithm applied to MDP $M$ must have a space-complexity bounded by $\widetilde{O}(SA)$, where $S$ and $A$ are the state-space and action-space sizes of $M$.
\end{definition}

The second definition has been utilized in previous studies to restrict algorithms from storing excessive amounts of information regarding the MDP dynamics within memory, distinguishing model-based RL \cite{4, 32}. Using these definitions, we present our model-free PAC-MDP algorithm, PDQL, in Section \ref{3}. From this, the problem being addressed in our study is formalized in the statement below.

\begin{problem} \label{p1}
    Let PDQL be a PAC-MDP algorithm. For the application of PDQL over $M$, for values $\epsilon, \delta \in (0,1)$, derive the greatest lower-bound on $t \in \mathbb{Z}^+$ such that $\mathbb{P}(|V_t(s) - V_*(s)| < \epsilon) > 1-\delta, \; \forall s \in \mathcal{S}$.
\end{problem} 

\subsection{Setting Formalization} \label{ass-sec}
Our work takes place in settings where the environment is endowed with a distance metric (e.g. Euclidean, Manhattan etc.). Our objective is to utilize properties of locality within these settings to make learning more efficient. Though our results strictly apply to these settings, we observe that many real-world problems are set in these environments (e.g. $\mathbb{R}^2$, $\mathbb{R}^3$, lattice etc.) in domains like navigation, control and healthcare \cite{robot1, robot2}. Further, there is extensive work on the application of RL within similar metric-based settings \cite{metric1, metric2, metric3, metric4}.


%

\begin{proposition} \label{ass1} 
    \textbf{(Distance Metric)} The state-space, $\mathcal{S}$, of $M$ is endowed with a distance metric, $\mathcal{D} : \mathcal{S} \times \mathcal{S} \rightarrow \mathbb{R}_{\geq 0}$.
    
    
\end{proposition}



 We use this proposition to further comment on locality and outline transition kernels within these environments.


\begin{proposition}\label{ass2} 
    \textbf{(Locality)} The transition kernel over any state $s \in \mathcal{S}$, is defined over the set of states that are within one unit distance of $s$. Denoting this set as $\mathcal{S}'_s := \{ s' \in \mathcal{S} \; \;| \;\; \mathcal{D}(s, s') \leq 1 \}$, with $|\mathcal{S}'_s| \leq A$.
\end{proposition}

We upper-bound the number of unique, possible transitions for any state in $M$ to establish a measure of locality using the size of the action-space. With Proposition \ref{ass2}, we are able to construct sub-MDPs using a subset of states within $M$ and still preserve the overall dynamics of $M$ within the region of the sub-MDP.



\section{Local Approximation of Value Functions}
Within this section, we establish the notion of sub-MDP and, through it, local approximations of value functions. Prior to defining a sub-MDP, we present the following Lemma.

\begin{lemma} \label{l1}
    For $\epsilon, \gamma \in (0,1)$, where $\gamma$ is the discount factor, if $T \geq \log_\gamma (\epsilon(1-\gamma))$, then $V_t(s) - V_t^T(s) \leq \epsilon, \; \forall t \in \mathbb{Z}^+, \forall s \in \mathcal{S}$.
\end{lemma}

\begin{proof}
Letting an infinite-horizon transition walk from state $s$ at timestep $t$ result in the maximum possible reward sequence, i.e. $\{1, 1, 1 ...\}$, the cumulative discounted reward is the geometric series $V_t(s) = \sum_0^\infty \gamma$. Noting that $V_t^T(s) = \sum_0^T \gamma$, for the condition $V_t(s) - V_t^T(s) \leq \epsilon$ to hold, the following must be true and solving this expression for $T$ yields the bound presented in the Lemma statement.
    \begin{equation*}
        \sum_0^\infty \gamma - \sum_0^T \gamma = \gamma^{T+1} + \gamma^{T+2} ... = \frac{\gamma^{T+1}}{1-\gamma}  \leq \epsilon
    \end{equation*}
\end{proof}

Using Lemma \ref{l1}, noting MDP $M$ is characterized by tuple $(\mathcal{S}, \mathcal{A}, \mathcal{T}, \mathcal{R}, \gamma)$, we define a sub-MDP $M_1$ within $M$ as follows.

\begin{definition} \label{d3}
    \textbf{(Sub-MDP)} A sub-MDP $M_1$ is constructed using parameters $s_1 \in \mathcal{S}$ and $\epsilon \in (0, 1)$. $M_1$ is centered on $s_1$ and is characterized by tuple $(\mathcal{S}_1, \mathcal{A}, \mathcal{T}, \mathcal{R}, \gamma))$. The state-space of sub-MDP $M_1$ is $\mathcal{S}_1 := \{ s \in \mathcal{S} \;\; | \;\; \mathcal{D}(s_1, s) < \log_\gamma (\epsilon(1-\gamma)) \}$. As $\mathcal{S}_1 \subset \mathcal{S}$, a transition defined in $M$ from a state within $\mathcal{S}_1$ to a state outside of $\mathcal{S}_1$ results in a self-loop within $M_1$.
\end{definition}

Per Proposition \ref{ass2}, the sub-MDP is centered on a state within $M$ and constructed using a breadth-wise expansion, denoted radius, of $T = \lceil \log_\gamma (\epsilon(1-\gamma)) \rceil$ transitions from the center state. Intuitively, the size of the sub-MDP $M_1$ determines how close the optimal policy trained over it is to the optimal policy trained over the MDP $M$ for state $s_1$.

\begin{remark} \label{rem1}
    \textbf{(Sub-MDP Size)} Let a sub-MDP be constructed using parameters $s_1 \in \mathcal{S}$ and $\epsilon \in (0,1)$. It will have a transition radius of $\lceil \log_\gamma (\epsilon(1-\gamma)) \rceil$ and its state-space size is bounded by $\lceil \log_\gamma ^{A} (\epsilon(1-\gamma)) \rceil $.
\end{remark}

The premise of our study is as follows, the further away a state $s_2 \in \mathcal{S}$ is from $s_1 \in \mathcal{S}$, the less bearing $V_*(s_2)$ has on $V_*(s_1)$. This is formalized in Lemma \ref{l1}, which determines the length of rollout required for finite value function approximation $V^T_*(s_1)$ to approximate $V_*(s_1)$. As such, an optimal value function, $V_{*1}(s)$, generalized over a sub-MDP $M_1$ and state-space $\mathcal{S}_1 \subset \mathcal{S}$, parametrized by the center state $s_1$ and radius $T = \lceil \log_\gamma (\epsilon(1-\gamma)) \rceil$, is $\epsilon$-optimal with respect to $M$ at $s_1$. Furthermore, $V_{*1}(s)$ is relatively less optimal with respect to $M$ further away from the center. This is formalized in Lemma \ref{l2}.

\begin{lemma} \label{l2}
     Let sub-MDP $M_1$ be constructed from MDP $M$ using a subset of states $\mathcal{S}_1 \subset \mathcal{S}$, centered on state $s_1 \in \mathcal{S}_1$, as in Definition~\ref{d3}. Letting $V_*(s)$ be the optimal value function over $M$ and $V_{*1}(s)$ be the optimal value function over $M_1$, $| V_{*1}(s) - V_*(s) | \leq \epsilon/\gamma^{\mathcal{D}(s_1, s)} $ for all $s \in \mathcal{S}_1$.
\end{lemma}

\begin{proof}
    Per Definition \ref{d3}, sub-MDP $M_1$ is constructed with a transition radius of $T = \lceil \log_\gamma (\epsilon(1-\gamma)) \rceil$. By Lemma \ref{l1}, the theoretical optimal value function that can be generalized over $M_1$ will have at most $\epsilon$ error at $s_1$. Noting that the transition radius determines the closeness of the value function over the sub-MDP to $V_*(s)$, the largest sub-MDP that can be constructed within $\mathcal{S}_1$ and centered on a state $s \in \mathcal{S}_1$ will have a transition radius of $T - \mathcal{D}(s_1, s)$. 
    
    Let $\epsilon_1 \in (0, 1)$ represent the error of the optimal value function $V_{*1}(s)$ over $M_1$ at a state $s \in \mathcal{S}_1$ when compared to the optimal value function $V_*(s)$ over $M$. Extending Lemma \ref{l1}, the following holds.

    \begin{align*}
        \log_\gamma (\epsilon(1-\gamma)) - \mathcal{D}(s_1, s) &\geq \log_\gamma (\epsilon_1(1-\gamma)) \\
        \frac{\epsilon(1-\gamma)}{\gamma ^ {\mathcal{D}(s_1, s)}} &\geq \epsilon_1(1-\gamma)
    \end{align*}
\end{proof}

With Lemma \ref{l2}, we establish the relation between the optimal value function over $M$, $V_*(s)$, and the optimal value function over its sub-MDP $M_1$, $V_{*1}(s)$, in terms of $\epsilon$ with respect to the state-space. Noting Remark \ref{rem1} and the size of each sub-MDP, we posit the sample-complexity of learning for each sub-MDP in Lemma \ref{l3}.

\begin{lemma} \label{l3}
    Let sub-MDP $M_1$ be constructed from MDP $M$, centered on state $s_1 \in \mathcal{S}$, as in Definition \ref{d3}. The sample-complexity of learning an $\epsilon$-optimal value function, $V_{t1}(s)$, 
    with high confidence within $M_1$ is asymptotically bounded by $\tilde{O}( A \log_\gamma ^{A} (\epsilon(1-\gamma)) \log (  A \log_\gamma ^{A} (\epsilon(1-\gamma))  ))$. 
\end{lemma}

\begin{proof}
    Let $M$ be a discrete MDP with a state-space size of $S$ and an action-space size of $A$. Further, let $\delta \in (0, 1)$ be the confidence parameter and $\gamma \in (0, 1)$ be the discount factor. The sample-complexity bound, proposed in \cite{27}, for model-free reinforcement learning over $M$, without Propositions \ref{ass1} and \ref{ass2}, is $O \left( \frac{SA}{\epsilon^2 (1-\gamma)^3}  \log(\frac{SA}{\delta(1-\gamma)}) \log(\frac{1}{\epsilon})  \right)$. Noting that the state-space size of a sub-MDP within $M$ is $\log_\gamma ^{A} (\epsilon(1-\gamma)) $, extending the previous sample-complexity bound to the sub-MDP results in a bound of 

    \tiny
    \begin{equation*}
        O \left( \frac{A \log_\gamma ^{A} (\epsilon(1-\gamma))}{\epsilon^2 (1-\gamma)^3}  \log \left( \frac{A \log_\gamma ^{A} (\epsilon(1-\gamma))}{\delta(1-\gamma)} \right) \log \left( \frac{1}{\epsilon} \right)  \right)
    \end{equation*}
    \normalsize
    
    \noindent which is simplified to the expression in the lemma statement.
\end{proof}

Lemma \ref{l3} provides a sample-complexity bound for learning $\epsilon$-optimal policies within sub-MDPs. Let $M_1$ be a sub-MDP of $M$, parametrized by central state $s_1 \in \mathcal{S}$ and error $\epsilon \in (0, 1)$, with a generalized $\epsilon$-optimal value function $V_{t1}(s)$. Noting Lemma \ref{l1}, i.e. $| V_{*1}(s_1) - V_*(s_1) | \leq \epsilon$, we get $| V_{t1}(s_1) - V_*(s_1) | \leq 2 \epsilon$. Lemma \ref{l3} further implies that, assuming a naive division of $M$ into $S$ sub-MDPs by constructing a sub-MDP at each state, it is possible to generalize a $2\epsilon$-optimal policy over $M$ within a fixed number of timesteps. This is formally stated in the following remark.

\begin{remark} \textbf{(Naive Sample-Complexity Bound)} \label{rem2}
    Given the size of $\mathcal{S}$ is $S$, assuming a sub-MDP is constructed centered on each state within $\mathcal{S}$, the sample-complexity of learning a $2\epsilon$-optimal function over $M$ is naively bounded by $\tilde{O} ( SA \log_\gamma ^{A} (\epsilon(1-\gamma))  \log (  A \log_\gamma ^{A} (\epsilon(1-\gamma))   ))$.
\end{remark}

\section{Sharper Asymptotic Bound on Sample-Complexity}

In this section, we derive a sample-complexity bound for reinforcement learning using local approximations that is sharper than the naive bound proposed in Remark \ref{rem2}. Noting the bound on sample-complexity presented in Lemma \ref{l3}, intuitively, the aim of this section is to determine the minimum number of sub-MDPs of $M$ that are required to generalize a near-optimal value function over $M$. 

Through Lemma \ref{l2}, we note that there is a correlation between the approximation error of the sub-MDP's optimal value function at a state to $V_*(s)$ and the distance from the state to the center of the sub-MDP. Subsequently, we note that multiple sub-MDPs within $M$ may overlap in state-space, and thus, \textit{we may obtain multiple sub-optimal estimates of the value function for states that are not sub-MDP centers}. Using Lemma \ref{l2}, though, we are aware of the degree of sub-optimality and can make use of multiple estimates of the value function at any state in $M$ to obtain a better approximation. In deriving sharper sample-complexity bounds, we utilize the following lemma.

\begin{lemma} \label{l4}
    Let $s_0 \in \mathcal{S}$ be a state in MDP $M$. Let $V_*(s_0)$ be the theoretical optimal value of $s_0$ with respect to $M$. Let $V_{*1}(s_0), V_{*2}(s_0) ... V_{*N}(s_0)$ be $\epsilon_1, \epsilon_2 ... \epsilon_N \in (0, 1)$ approximations of $V_*(s_0)$, respectively. Then, $\mathbb{P} \left( | \frac{1}{N}\sum_{i=1}^N V_{*i}(s_0) - V_*(s_0) | \geq \epsilon \right) \leq 2e^{ \frac{-2N^2\epsilon^2}{\sum_{i=i}^N 2\epsilon_i}  }$.
\end{lemma}

\begin{proof}
    \textit{Sketch.} This result is achieved through a direct application of Hoeffding's Inequality \cite{hoeff}.
\end{proof}

We use Lemma \ref{l4} in determining the minimum overlap required by sub-MDPs in order to generalize an $\epsilon$-optimal value function over $M$ with high confidence. This allows us to determine the minimum number of sub-MDPs required. Let $M_1, M_2 ... M_N$ represent sub-MDPs of $M$ that are constructed using arbitrary centers and error parameter $\epsilon \in (0, 1)$ over state-spaces that are subsets of $\mathcal{S}$, i.e. $\mathcal{S}_1, \mathcal{S}_2 ... \mathcal{S}_N \subset \mathcal{S}$, respectively. Let $V_{*1}(s), V_{*2}(s) ... V_{*N}(s)$ be the optimal value functions over the respective sub-MDPs. Let $s_0 \in \mathcal{S}_1 \cap \mathcal{S}_2 \cap ... \cap \mathcal{S}_N$ be a common state within all sub-MDPs and $\epsilon_1, \epsilon_2 ... \epsilon_N$ be the respective errors between $V_*(s_0)$ and $V_{*1}(s_0), V_{*2}(s_0) ... V_{*N}(s_0)$. A direct application of Lemma \ref{l4} allows us to determine the confidence with which the value function generalized through local approximation is $\epsilon$-optimal. Within Lemma \ref{l5}, we derive the amount of overlap in state-spaces required by sub-MDPs to achieve the $\epsilon$-optimal result with a confidence given by $\delta \in (0,1)$.

\begin{lemma} \label{l5}
    Let $\epsilon, \delta \in (0, 1)$ represent the error and confidence parameters. For all states $s \in \mathcal{S}$, if there are at least $\frac{2}{\epsilon} \log \frac{2S}{\delta}$ sub-MDP centers within $\lceil \log_\gamma 0.5 \rceil$ transitions of $s$, an $\epsilon$-optimal value function over $M$ can be generalized with $1-\delta$ confidence using the sub-MDP value function estimates.
\end{lemma}

\begin{proof}
    Let $V_*(s)$ be the optimal value function over $M$. Let $M_1, M_2 ... M_N$ be sub-MDPs of $M$ parametrized by centers $s_1, s_2 ... s_N \in \mathcal{S}$ and approximation error $\epsilon \in (0, 1)$. Let $V_{*1}(s), V_{*2}(s) ... V_{*N}(s)$ represent the optimal value functions over the respective sub-MDPs. For a state $s_0 \in \mathcal{S}$, if $| V_{*i}(s_0) - V_*(s_0) | \leq 2\epsilon$, then $\mathcal{D}(s_0, s_i) \leq \log_\gamma 0.5$ by Lemma \ref{l1} for all $i \in \{1, 2 ... N\}$, as the transition radius from $s_0$ to the boundary of $M_i$ permits the learning of a $2\epsilon$-optimal value function.

    Noting that the size of $\mathcal{S}$ is $S$, we use Lemma \ref{l4} to bound the probability that the difference between $\frac{1}{N} \sum_{i=1}^N V_{*i}(s_0)$ and $V_*(s_0)$ is greater than $\epsilon$ in the following expression by substituting $2\epsilon$ for all $\epsilon_i$.

    \begin{equation*}
        \mathbb{P} \left( \left| \frac{1}{N}\sum_{i=1}^N V_{*i}(s_0) - V_*(s_0) \right| \geq \epsilon  \right) \leq 2e^{ \frac{-N^2\epsilon}{2} } \leq \frac{\delta}{S}
    \end{equation*}

    \noindent Resolving this, we obtain the following result.

    \begin{equation*}
        N \geq \frac{2}{\epsilon} \log \frac{2S}{\delta}
    \end{equation*}
\end{proof}

The implication of Lemma \ref{l5} is that the number of sub-MDP overlaps for any state scales logarithmically with the state-space size. Furthermore, for a fixed value of $\epsilon$, the confidence of the generalized value function over $M$ increases with the sub-MDP overlap within the state-space. However, we still lack a correlation between the size of the state-action space and the number of sub-MDPs required. Intuitively, the number of sub-MDPs required increases with the size of the state-space as there are more states to be covered. However, for a fixed state-space size, an increase in the action-space size increases the connectivity of MDP $M$, thereby decreasing the transition distance between states. Consequently, an increase in the action-space increases the state-space size of a sub-MDP for the same transition radius and decreases the overall number of sub-MDPs required within $M$ to provide a value function estimate. This relation is formalized in Lemma \ref{l6}.

\begin{lemma} \label{l6}
    Let $\epsilon \in (0, 1)$ represent the construction error for sub-MDPs $M_1, M_2 ... M_L$ constructed in $M$. Let $\mathcal{S}_1, \mathcal{S}_2 ... \mathcal{S}_L$ be the respective state-spaces of the sub-MDPs. Further, let $\mathcal{S}_1 \cup \mathcal{S}_2 \cup ... \cup \mathcal{S}_L = \mathcal{S}$. The value of $L$ is lower-bounded by $\frac{2S}{\epsilon \log_\gamma ^ {A} \epsilon(1-\gamma)} \log \frac{2}{\delta}$.
\end{lemma}


\begin{proof}
    As a consequence of Propositions \ref{ass1} and \ref{ass2}, the coverage of states in $M$ by the sub-MDPs is maximized when the sub-MDP centers are distributed uniformly over $\mathcal{S}$. Letting the action-space size of $M$ and, subsequently, $M_1, M_2 ... M_L$, be $A$, noting Remark \ref{rem1}, the state-space size of each individual sub-MDP is upper bounded by $\lceil \log_\gamma ^ {A} \epsilon(1-\gamma) \rceil$. Noting that the overall number of value estimates made is $L$ multiplied by the sub-MDP size, we establish the following bound relating the number of estimates with $N$, the number of per-state estimates.

    \begin{equation*}
        L \log_\gamma ^ {A} \epsilon(1-\gamma) > S N
    \end{equation*}

    \noindent Equivalently,

    \begin{equation*}
        L \log_\gamma ^ {A} \epsilon(1-\gamma) > S \frac{2}{\epsilon} \log \frac{2S}{\delta}
    \end{equation*}

    \noindent Noting that there are $L$ sub-MDP centers, i.e. states that do not require repeat estimation via overlap, the following adjustment to the bound is made.

    \begin{equation*}
        L \log_\gamma ^ {A} \epsilon(1-\gamma) > S \frac{2}{\epsilon} \log \frac{2S}{\delta} - L
    \end{equation*}

    \noindent We relax the inequality by letting $\log S$ be a lower-bound on $L$ in the RHS. Solving for $L$ yields the bound in the lemma statement.   
\end{proof}

Using Lemma \ref{l5}, which indicates the number of times a state's value needs to be independently determined within $M$, we establish Lemma \ref{l6}, which indicates the number of sub-MDPs required in an MDP with state-space size $S$ and action-space size $A$ to achieve this coverage. We further posit that the overall sample-complexity required to learn an $2\epsilon$-optimal value function over $M$ with $1-\delta$ confidence is the sub-MDP sample-complexity, established in Lemma \ref{l3}, multiplied by  $\frac{2S}{\epsilon \log_\gamma ^ {A} \epsilon(1-\gamma)} \log \frac{2}{\delta}$. This result is intuitive as the number of sub-MDPs multiplied by the sample-complexity required for each sub-MDP yields the overall sample-complexity. We formalize this in the following theorem.



\begin{theorem} \textbf{(Main Result)} \label{t1}
    Let $\epsilon, \delta, \gamma \in (0, 1)$ represent error, confidence and discount factor. For an MDP $M$ with state-space size $S$ and action-space size $A$, by locally approximating $M$ using sub-MDPs, an $\epsilon$-optimal value function will be generalized within $O \left( \frac{SA}{\epsilon^3 (1-\gamma)^3} \log \left( \frac{A}{\delta(1-\gamma)} \right)  \log \frac{1}{\epsilon} \log \frac{1}{\delta})  \right)$ timesteps with $1-\delta$ confidence.
\end{theorem}

\begin{proof}
    The proof of this theorem follows from Lemmas \ref{l3} and \ref{l6}. We determine the sample-complexity of learning a $2\epsilon$-optimal value function as an asymptotic bound in $O(LC)$, where $L$ is the number of sub-MDPs required to be constructed and $C$ is the sample-complexity of learning for each sub-MDP. We have determined the number of sub-MDPs required to approximate the larger MDP $M$ to be bounded by $\frac{2S}{\epsilon \log_\gamma ^ {A} \epsilon(1-\gamma)} \log \frac{2}{\delta}$. By Lemma \ref{l3}, we have determined the sample-complexity of each sub-MDP to be $O \left( \frac{A \log_\gamma ^{A} (\epsilon(1-\gamma))}{\epsilon^2 (1-\gamma)^3}  \log \left( \frac{A \log_\gamma ^{A} (\epsilon(1-\gamma))}{\delta(1-\gamma)} \right) \log \left( \frac{1}{\epsilon} \right)  \right)$. The multiplication of the two terms results in the asymptotic bound presented in the theorem statement.

    Note that the sample-complexity required to generalize an $\epsilon$-optimal value function, as opposed to it being $2\epsilon$-optimal, requires substituting $\epsilon/2$ into the lemma statement. This introduces a constant factor that is disregarded when taking the asymptotic bound.
\end{proof}

Theorem \ref{t1} presents a theoretical bound on the sample-complexity of learning. We construct an algorithm in Section \ref{3} that makes use of this result to efficiently learn a value function.
\section{The PDQL Algorithm} \label{3}

\begin{algorithm*} [t]
    \caption{Probabilistic Delayed Q-Learning} \label{alg2}

    \vspace{-0.3cm}
    \begin{multicols}{2}
    
    \begin{algorithmic}
        \STATE \textbf{Inputs}:
    \end{algorithmic}
    \begin{algorithmic} [1]
        \STATE State-Space: $\mathcal{S}$
        \STATE Action-Space: $\mathcal{A}$
        \STATE Transition Function: $\mathcal{T}$
        \STATE Reward Function: $\mathcal{R}$
    \end{algorithmic}

    \begin{algorithmic}
        \STATE \textbf{Parameters}: 
    \end{algorithmic}
    \begin{algorithmic} [1]
        \STATE Discount Factor: $\gamma$
        \STATE Sampling Parametre: $q$
        \STATE Error Bound: $\epsilon$
        \STATE Transition Radius: $T$
    \end{algorithmic}

    \begin{algorithmic}
        \STATE \textbf{Initialize Values}:
    \end{algorithmic}
    \begin{algorithmic}[1] 
        \FOR {all values $(s, a) \in \mathcal{S} \times \mathcal{A}$}
        \STATE $Q(s, a) \leftarrow \frac{1}{1-\gamma}$ \;\;\;\;\;\;\;\;\;\;\;\;\;\;\;\;\;\;\;\;\;\;\;\;\;\;\;\;\;\; \textit{\%current Q-Value}
        \STATE $U(s, a) \leftarrow 0$ \;\;\;\;\;\;\;\;\;\;\;\;\;\;\;\;\;\;\;\;\;\;\;\;\;\;\;\;\;\;\;\;\;\;\; \textit{\%update attempt}
        \STATE $C(s, a) \leftarrow 0$ \;\;\;\;\;\;\;\;\;\;\;\;\;\;\;\;\;\;\;\;\;\;\;\;\;\;\;\;\;\;\;\;\;\;\; \textit{\%visit counter}
        \STATE $UNLOCK(s, a) \leftarrow True$ \;\;\;\;\;\;\;\;\;\;\;\;\;\; \textit{\%learning lock}

        \ENDFOR
    \end{algorithmic}

    \begin{algorithmic}
        \STATE \textbf{Note}: Let $R$ be the current state reward and $s'$ be the subsequent state after executing $a$ on $s$ under $\mathcal{T}$
    \end{algorithmic}

    \begin{algorithmic}
        \STATE \textbf{Output}:
    \end{algorithmic}
    \begin{algorithmic}[1] 
        \WHILE{there exists $(s, a) \;|\; UNLOCK(s, a)=True$}
        \STATE Select state-action \\$(s, a) \;|\; a=argmax_{a \in \mathcal{A}}Q(s, a)$
        \STATE $U(s, a) \leftarrow U(s, a) + R + \gamma Q(s', a')$

        
        \STATE $C(s, a) \leftarrow C(s, a)+1$,
        
        \IF {$C(s, a) = q$}
        \STATE $C(s, a) \leftarrow 0$,
        
        \IF {$Q(s, a) - U(s, a)/q \geq 2\epsilon$}
        \STATE $Q(s, a) \leftarrow U(s, a)/q + \epsilon$
        
        \FORALL{$s' \;|\;  \mathcal{D}(s, s') < \lceil \log_\gamma (\epsilon(1-\gamma)) \rceil $}
            \FORALL{$a' \in \mathcal{A}$}
            \STATE $UNLOCK(s', a') \leftarrow True$
            \ENDFOR
        \ENDFOR
        
        \ELSE
        \STATE $UNLOCK(s, a) \leftarrow False$
        \ENDIF

        \ENDIF

        \STATE $U(s, a) \leftarrow 0$
        \ENDWHILE
        
    \end{algorithmic}

    \end{multicols}

    \vspace{-0.4cm}
\end{algorithm*}


Within this section, we describe a model-free RL algorithm, Probabilistic Delayed Q-Learning (PDQL), that utilizes local approximation in generalizing a $2\epsilon$-optimal value function. We extend results from existing algorithms, DQL \cite{16} and VRQL \cite{27} using the results in the prior sections. Intuitively, there are two characteristics of PDQL that distinguish it from Q-Learning: the \textit{update condition} and the \textit{lock condition}. PDQL is detailed in Algorithm \ref{alg2}.

\subsection{The Update Condition}


The algorithm begins by initializing all state-action Q-values to $1/(1-\gamma)$, where $\gamma \in (0, 1)$ is the discount factor. For $\epsilon, \delta \in (0, 1)$, the update condition for PDQL is as follows, where $\mathbb{E}_T[Q_{t}(s, a)]$ is the expectation of the $T$-step Q-value computed over the distribution of possible transition walks from $(s, a)$ after $t$ timesteps. Letting $T$ correspond with sub-MDP radius $\lceil \log_\gamma \epsilon(1-\gamma) \rceil$ per Lemma \ref{l1}, if the following inequality is satisfied, the stored Q-value is updated to the new expected Q-value.

\begin{equation} \label{eq1}
    Q_t(s, a) - \mathbb{E}_T[Q_{t}(s, a)] \geq 2\epsilon
\end{equation}

The rationale for this condition follows from \cite{16}. When generalizing an $\epsilon$-optimal policy, to prevent instability, it is sufficient to update a Q-value if the difference between the current value and the expected value is significant enough. In order to estimate the expected Q-value, we sample the Q-value of a state $q$ times and average it when making the update. The update step in line 7 of Algorithm \ref{alg2} is formalized as follows, where $V^T_t(s)$ is the $T$-step value function at timestep $t$ and $s'_i$ is the state transitioned to on the i-th execution of $(s, a)$.

\begin{equation} \label{eq2}
    Q_{t+1}(s, a) := \frac{1}{q} \sum\limits_{i=1}^q \left(  r_i(s,a) + \gamma V^T_{i}(s'_i) \right) + \epsilon,
\end{equation}


Note that we add an $\epsilon$ to the update to ensure that $Q_t(s,a) > Q_*(s,a)$ for all $t \geq 0$. Assuming $q$ is large enough, then with $\delta$ confidence, the stored value of $Q_{t+1}(s, a)$ is within $\epsilon$ of the expected value. The following lemma indicates the lower-bound of $q$ for this to occur.

\begin{lemma} \label{l7}
    For $\epsilon, \delta \in (0, 1)$, if the following bound on $q$ holds, then $Q_t(s,a) - \mathbb{E}_T[Q_t(s,a)] < \epsilon$ holds for all $(s, a) \in \mathcal{S} \times \mathcal{A}$ and all timesteps $t \geq 0$ with $1-\delta$ confidence.

    \begin{equation*}
        q \geq \frac{\log \left(  \frac{2SA}{\epsilon}\log\frac{2}{\delta} \left( \frac{1}{\log_\gamma^{A} \epsilon(1-\gamma) } + \frac{A}{\epsilon(1-\gamma)} \right)  \right) }{2\epsilon^2(1-\gamma)^2}
    \end{equation*}
\end{lemma}

\begin{proof}
    Note that each state-action's Q-value can be updated at most $\frac{1}{\epsilon(1-\gamma)}$ times due to the update condition requiring that there is a difference of $\epsilon$ between each update and the initialization of each Q-value to $\frac{1}{1-\gamma}$. As such, there can be at most $\frac{A\log_\gamma^{A} \epsilon(1-\gamma)}{\epsilon(1-\gamma)}$ successful Q-value updates made within a sub-MDP.
    
    The remainder of this proof utilizes Subsection \ref{locksec}. Further noting the locking condition and Lemma \ref{unlockrad}, in the worst case there can be at most $AL \left( 1+\frac{A\log_\gamma^{A} \epsilon(1-\gamma)}{\epsilon(1-\gamma)} \right)$ attempted updates made, where $L$ is the lower-bound on the number of sub-MDPs posited by Lemma \ref{l6}. This scenario occurs when an update is attempted, only for the condition in \eqref{eq1} to not be met, resulting in an unsuccessful update attempt and the state being \textit{locked}. If further updates \textit{are possible}, regardless of the lock being placed, there may be up to $AL$ updates made before the state and its associated sub-MDP are unlocked.

    As such, we use an application of Hoeffding's Inequality to yield the following inequalities, bounding the difference between the expected Q-value and mean of encountered Q-values with high confidence. Solving for $q$ yields the result in the lemma statement.

    \begin{align*}
        \mathbb{P} \left( Q_t(s,a) - \mathbb{E}_T[Q_t(s,a)] > \epsilon \right) < \exp \left(  -2q\epsilon^2(1-\gamma)^2 \right) \\
        \frac{\delta}{AL \left( 1+\frac{A\log_\gamma^{A} \epsilon(1-\gamma)}{\epsilon(1-\gamma)} \right)} < \exp \left(  -2q\epsilon^2(1-\gamma)^2 \right)
    \end{align*}
\end{proof}

This lower-bound on sampling size is corroborated by VRQL and is needed to be held in order for the sample-complexity of learning for each sub-MDP to be equivalent to the asymptotic bound presented in Lemma \ref{l3}, derived from \cite{27}. If a state-action $(s,a)$ is sampled $q$ times and the update condition in \eqref{eq1} is not met, the state $s$ is \textit{locked} and further updates are not to take place. We elaborate on this in the following subsection.


\begin{figure*}[tp]
\label{fig1}
\centering
\begin{subfigure}[t]{0.4\linewidth} \label{figa}
    \centering
    \includegraphics[width=\linewidth]{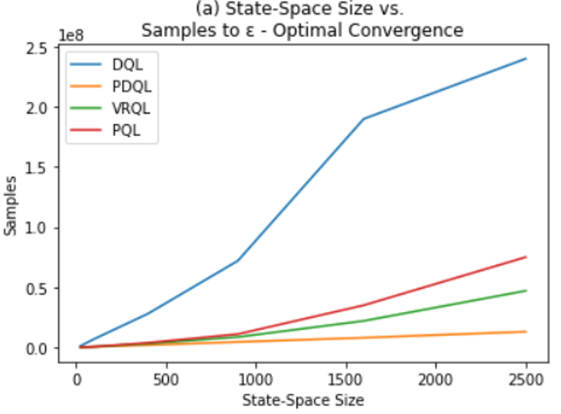}
    \caption{Average growth of convergence rate with respect to the state-space size, where the algorithm is said to have converged upon reaching $\epsilon$-optimality. }
\end{subfigure}\hfil
\begin{subfigure}[t]{0.42\linewidth} \label{figb}
    \centering
    \includegraphics[width=\linewidth]{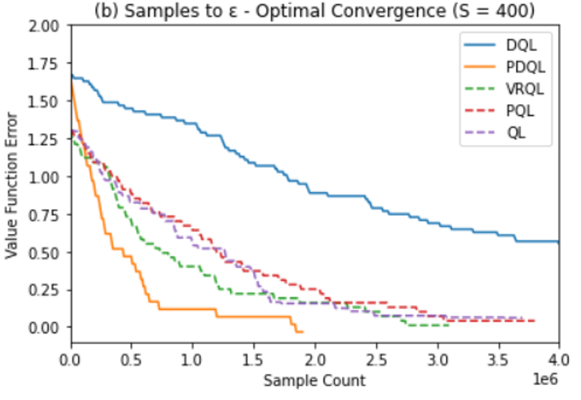}
    \caption[b]{Comparison of convergence rates between algorithms (Q-Learning, DQL, PQL, VRQL and PDQL) with respect to sample count, $t$, where mean error is measured as the expected difference between the current value function and the optimal value function, i.e. $\mathbb{E}_s[V_t(s) - V_*(s)]$. }
 \end{subfigure}
\vspace{0.3cm}
 \hrule
\end{figure*}


    

\subsection{The Lock Condition} \label{locksec}
The primary method by which we apply local approximation in the formulation of PDQL is through the \textit{locking/unlocking mechanism} of PDQL. For each state, PDQL assigns a boolean variable that indicates whether further updates are possible to its value. Following an attempted update, if the condition in \eqref{eq1} is not met, the update is not carried out and the state is \textit{locked} - barring it from future updates. The locking mechanism has two purposes: firstly, it is used to indicate convergence, i.e. when all states are locked, the algorithm has converged; secondly, it allows the algorithm to converge faster by focusing effort toward the \textit{unlocked} states.

Per Lemma \ref{l1}, for values $\epsilon, \gamma \in (0, 1)$, an update to a Q-value of a state $s_1 \in \mathcal{S}$ can affect the Q-values of the states within $\lceil  \log_\gamma \epsilon(1-\gamma)  \rceil$ transitions of $s_1$. This corresponds with Definition \ref{d3}, which defines sub-MDPs. We formulate the following lemma regarding this.

\begin{lemma} \textbf{(Unlocking Radius)} \label{unlockrad}
    If a state's value is affected by more than $\epsilon$ from an update to a state $s_1 \in \mathcal{S}$, it must lie within the state-space of a sub-MDP centered on $s_1$ with a radius of $\lceil  \log_\gamma \epsilon(1-\gamma)  \rceil$ transitions.
\end{lemma}

\begin{proof}
    \textit{Sketch.} This lemma is a consequence of Lemma \ref{l1}.
\end{proof}

With regard to Lemma \ref{unlockrad}, noting that an update on a state-action may result in updates to state-action Q-values in its vicinity, all states in the vicinity of $s_1$ must be subsequently \textit{unlocked} to generalize an $\epsilon$-optimal value function. That is, for the sub-MDP constructed within the state-space $\mathcal{S}_1$, where $\mathcal{S}_1 := \{ s \in \mathcal{S} \;\; | \;\; \mathcal{D}(s_1, s) < \log_\gamma (\epsilon(1-\gamma)) \}$, all $s \in \mathcal{S}_1$ are \textit{unlocked} within PDQL. This is conducted in line 9 of Algorithm \ref{alg2}.

Intuitively, the \textit{unlocking} procedure is a method of constructing sub-MDPs within $M$ based on the states that probably require updates. With minimal overlap, Lemma \ref{l6} indicates how many sub-MDPs can be unlocked simultaneously. Similarly, the \textit{locking} procedure diverts the algorithm's efforts to areas of the state-space that are more likely to require attention, where Lemmas \ref{l5} and \ref{l7} indicate precisely how much effort is required to be provided. With this, PDQL demonstrates the asymptotic sample-complexity derived in Theorem \ref{t1}.

The locking procedure in our work is in direct contrast with DQL \cite{16}, where a \textit{global unlocking} procedure is implemented instead of a \textit{local unlocking} resulting in greater sample-efficiency.

\subsection{Space-Complexity Analysis of PDQL}
Referring to Definition \ref{d2} for the space-complexity bounds on model-free RL algorithms, we briefly show PDQL is within the bound of $\widetilde{O}(S A)$. Algorithm \ref{alg2} stores the following: a record of up to $q$ Q-values for each state-action $(s, a) \in \mathcal{S} \times \mathcal{A}$, $O(SA)$; and a lock variable for each state, $O(S)$. It is realistic to assume the existence of an oracle that can be queried to yield the minimum number of transitions between two states due to the distance metric placed by Proposition \ref{ass1}. However, in its absence, a minor adjustment to the algorithm can be made such that the sub-MDP records are stored to facilitate the unlocking procedure, requiring the following space-complexity $O(L \log_\gamma ^ {A} \epsilon(1-\gamma))$. Simplifying this, the asymptotic space-complexity of PDQL is $O(qSA + S + \frac{2S}{\epsilon}\log\frac{2}{\delta})$, which is equivalent to $\widetilde{O}(S A)$.

\section{Experiments} \label{exp}

In order to demonstrate the results of this study, we have conducted experiments to observe the convergence rate of PDQL in various environments against similar tabular PAC-MDP algorithms. As our primary result is the removal of the logarithmic dependency on the state-space size, we deploy our learning algorithm in environments of different sizes to show the growth in the convergence rate of PQDL. We demonstrate our results in discretized versions of the \textit{Lunar Lander} environment from \textit{Gymnasium} \cite{gym}.

The parameters of our experiment are as follows: $\gamma = 0.9$; $\delta = 0.001$; $\epsilon = 0.01$; $S = \{ 50, 200, 500, 1000, 1500, 2000 \}$; with the state-wise reward scaled to the range $[0, 1]$. Considering that this algorithm is based on Q-Learning, this paper shows its comparison with Q-Learning, Delayed Q-Learning (DQL)\cite{16}, Phased Q-Learning (PQL) \cite{37} and Variance Reduced Q-Learning (VRQL) \cite{27}.



    
    
    
    
    

There are two characteristics of the data that we comment on. The first is the rate of convergence displayed in Figure b, which shows that PDQL has a faster rate of convergence than the other forms of Q-Learning. This also confirms the proposition that DQL and VRQL are conservative in the number of samples required prior to an update, noting its relatively slow rate of convergence when compared with the other algorithms. This was observed during all the trials recorded. The second is the rate of growth in the number of samples to convergence with respect to the size of the state-space, which scales more efficiently in $SA\log A$ as the bound suggests for PDQL than $SA \log(SA)$ for the remaining algorithms (Figure a).


\section{Conclusion}
The objective of our study has been to sharpen sample-complexity bounds of model-free RL under Propositions \ref{ass1} and \ref{ass2}. We do so by constructing local approximations of the larger MDP and learning the value function locally. We then combine the learnt values alongside the associated error to approximate the value function across the larger MDP with high confidence. The main result that we present is that an $\epsilon$-optimal value function can by generalized within $O \left( SA \log A  \right)$ timesteps with $1-\delta$ confidence. As such, the algorithm that we present in this paper, PDQL, demonstrates faster convergence than previous works.




\begin{ack}
This research is supported by the National Research Foundation, Singapore and DSO National Laboratories under the AI Singapore Programme (AISG Award No: AISG2-RP-2020-017). This research is also supported by the National Research Foundation, Prime Minister’s Office, Singapore under its Campus for Research Excellence and Technological Enterprise (CREATE) programme through the programme DesCartes. This research was also supported by MoE, Singapore, through the Tier-2 grant MOE2019-T2-2-040. 
\end{ack}



\bibliography{mybibfile}

\end{document}